\documentclass[11pt]{article}

\usepackage[table]{xcolor}

\definecolor{oursrow}{HTML}{F3F8FF}      % very light blue
\definecolor{bestcell}{HTML}{DFF3E3}     % light green
\definecolor{secondcell}{HTML}{FFF3CD}   % light yellow

\newcommand{\best}[1]{\cellcolor{bestcell}\textbf{#1}}
\newcommand{\second}[1]{\cellcolor{secondcell}\underline{#1}}

\usepackage[final]{acl}

\usepackage{times}
\usepackage{latexsym}
\usepackage[T1]{fontenc}
\usepackage[utf8]{inputenc}
\usepackage{microtype}
\usepackage{inconsolata}

\usepackage{graphicx}
\usepackage{amsmath}
\usepackage{amssymb}
\usepackage{amsthm}
\usepackage{multirow}
\usepackage{float}
\usepackage{booktabs}
\usepackage{colortbl}
\usepackage{stfloats}  % allow figure*/table* to land at bottom of page (not just top)
\usepackage{placeins} % \FloatBarrier to keep wide floats inside their sections

\usepackage{algpseudocode}

\newtheorem{proposition}{Proposition}

\newcommand{\NR}{\textcolor{black!40}{--}}
\newcommand{\sys}{GuardianAgent}
\newcommand{\amrsf}{AMRSF}
\newcommand{\ours}{Ours}

\title{\sys: Policy-Conditioned Risk-Adaptive Anonymization with Verified Adversarial Escalation}

\author{
  Ruiyi Yang\textsuperscript{1} \quad
  Gayathri Lihinikaduarachchi\textsuperscript{1} \quad
  Rahat Masood\textsuperscript{1} \\
  \textbf{Flora D. Salim}\textsuperscript{1} \quad
  \textbf{Salil S. Kanhere}\textsuperscript{1} \\ 
  \textsuperscript{1}School of Computer Science and Engineering,
  UNSW Sydney, Australia \\
  \texttt{\{ruiyi.yang, g.lihinikadu\_arachchillage\}@student.unsw.edu.au} \\  \texttt{\{rahat.masood, flora.salim, salil.kanhere\}@unsw.edu.au}
}

\begin{document}
\maketitle

% =====================================================================
\begin{abstract}
Privacy protection for live web traffic requires more than detecting private spans. Agent-based privacy protection systems must determine whether an outgoing action complies with the destination site's privacy policy, then apply only the level of rewriting or sanitisation justified by the residual disclosure risk. We present \textbf{\sys}, a policy-conditioned anonymization framework that couples structured risk assessment with verified adaptive rewriting. \sys{} computes risk through \textbf{\amrsf{}} (Adaptive Multi-factor Risk Scoring Formula), an explicit controller that combines policy-violation likelihood with data sensitivity, recipient transmission, purpose legitimacy, contextual basis, and policy transparency, rather than relying on an LLM to assign risk directly. This risk score determines both the allow/transform/deny decision and the initial anonymization level. For efficiency, \sys{} uses an evidential fast path for low-uncertainty policy matches and invokes an LLM slow path only for uncertain cases. For rewriting, it applies a \textbf{five-level hierarchy} driven by a \textbf{verified adversarial guesser}: guesses trigger escalation only when supported by the original text, preventing hallucinated attacker confidence from causing unnecessary over-anonymization. Experiments across three benchmarks spanning legal text (TAB), Reddit posts (SynthPAI), and multi-format synthetic PII records (PII-Masking-300k) show that \sys{} achieves the strongest privacy--utility trade-off among published baselines and is the only method to reach $\geq\!0.90$ privacy in all three domains, remaining robust under a backbone switch. Action-context stress tests further show that the same outgoing text receives different decisions and anonymization strengths under different recipients, purposes, action bases, and policy-transparency conditions.\footnote{Code and evaluation harness at \url{https://anonymous.4open.science/r/guardianAgent/}.}
\end{abstract}

% =====================================================================
\section{Introduction}
\label{sec:intro}

Modern web applications routinely expose fine-grained personal signals through form fields, paste events, outbound requests, clipboard content, and chatbot interactions. A browser-based \emph{privacy agent} should therefore intervene before an action leaves the browser: it must determine whether the action is consistent with the destination site's privacy policy, and when needed, rewrite the outgoing content to reduce privacy leakage~\citep{zhou2025rescriber, staab2025language, chen2023hide, mireshghallah2023can}. This setting brings together two NLP problems that are usually studied in isolation: \textbf{policy matching}, which grounds an observed action in a heterogeneous privacy policy, and \textbf{text anonymization}, which rewrites content so that sensitive attributes cannot be inferred.

Existing work leaves three gaps. First, policy-understanding systems classify, retrieve, or summarize privacy-policy statements~\citep{harkous2018polisis,ahmad2021intent,srinath2021privaseer,tang2023policygpt,sun2025empowering}, and recent LLM-based interfaces make policies easier to query and assess~\citep{chen2025clear,freiberger2026helping}. However, these systems do not provide an operational risk signal that links live action context to the required strength of anonymization. The same text may be acceptable with light rewriting when sent to a first-party service for core functionality, but require stronger rewriting or blocking when sent to an advertising network for profiling. Second, recent LLM-based anonymizers often invoke a large model for every browser event~\citep{staab2025language,shao2025agentstealth}, making them poorly suited for low-latency browser-resident mediation. Third, adversarial anonymizers commonly escalate rewriting whenever an auxiliary attacker reports high confidence. Such confidence, however, may reflect hallucinated priors rather than genuine residual leakage. For example, given \texttt{[PERSON] going to [LOCATION] this afternoon}, an attacker may confidently guess ``coffee shop'' although the original was ``school with Bob'', causing unnecessary over-anonymization.

This paper presents \textbf{\sys}, a policy-conditioned privacy agent that couples structured risk assessment with verified adaptive anonymization. \sys{} computes risk through \amrsf{}, an explicit controller that combines policy-violation likelihood with data sensitivity, recipient transmission, purpose legitimacy, contextual basis, and policy transparency. The resulting score jointly determines the allow/transform/deny decision and the initial anonymization level, ensuring that rewriting strength is controlled by action context rather than by an LLM-assigned risk judgment alone. To reduce latency, \sys{} uses an evidential fast path for low-uncertainty policy matches and routes only uncertain cases to an LLM slow path. The anonymizer then applies a five-level rewrite hierarchy and escalates only when a verified adversarial guesser identifies residual leakage that is actually supported by the original text. Our main contributions are:

\begin{itemize}\itemsep -1pt
\item \textbf{Policy-conditioned risk-adaptive anonymization}: we introduce a structured risk score that jointly governs action-level decisions and anonymization strength. Action-context stress tests show that the same outgoing text receives different decisions and initial anonymization levels depending on recipient, purpose, action-basis, and policy transparency.

\item \textbf{Verified adversarial anonymization}: we propose a hierarchical anonymizer that escalates rewriting only when adversarial guesses are supported by the original text, preventing hallucinated confidence from triggering unnecessary semantic distortion.

\item \textbf{Cross-domain empirical validation}: experiments on legal text (TAB), Reddit posts (SynthPAI), and multi-format PII records (PII-Masking-300k) show that \sys{} consistently improves the privacy--utility trade-off over published baselines, is the only method to reach $\geq\!0.90$ privacy across all three domains, and maintains its gains under a backbone switch.

\end{itemize}

% =====================================================================
\section{Related Work}
\label{sec:related}

\paragraph{Privacy-policy grounding.}
Prior work has developed methods for extracting structured statements from privacy policies to enable downstream systems to reason about data practices. Polisis~\citep{harkous2018polisis} trained classifiers on the OPP-115 schema~\citep{wilson2016creation}; PrivaSeer~\citep{srinath2021privaseer} built a corpus-scale search infrastructure for privacy policies; PrivBERT~\citep{srinath2021privacy} and PolicyGPT~\citep{tang2023policygpt} studied domain-adaptive and LLM-based policy understanding; and \citet{ahmad2021intent} framed policy analysis as joint intent classification and slot filling. These systems make privacy policies more machine-interpretable, but they primarily support descriptive analysis: they do not directly ground a live browser action in the relevant policy evidence, nor do they determine whether the outgoing content should be allowed, rewritten, or blocked.

\paragraph{Privacy-risk frameworks.}
Quantitative privacy-risk modelling draws on regulatory tiering and contextual theories, including NIST SP~800-122~\citep{mccallister2010guide} and GDPR special categories~\citep{regulation2016regulation}, FAIR-style likelihood--impact products~\citep{cronk2021quantitative}, and contextual integrity~\citep{nissenbaum2009privacy,mireshghallah2023can}. Recent LLM-based systems further support privacy-policy assessment and risk comprehension~\citep{sun2025empowering,xie2025evaluating}, while agent-oriented work studies privacy risks in interactive LLM systems~\citep{ngong2025protecting}. These works are mainly used for explanation, auditing, or evaluation; \sys{} instead uses risk factors as a runtime controller linking policy evidence and action context to mediation decisions and anonymization strength.

\paragraph{LLM-based text anonymization.}
Recent anonymization methods increasingly rewrite content rather than only tagging sensitive spans. Staab~\textit{et al.}\@'s adversarial anonymizer~\citep{staab2025language} iteratively rewrites text until a simulated attacker cannot infer attributes; AgentStealth~\citep{shao2025agentstealth} distils this adversarial loop with reinforcement learning; HaS~\citep{chen2023hide} sanitizes prompts before third-party LLM calls; CONFAIDE~\citep{mireshghallah2023can} applies contextual-integrity prompts; and Rescriber~\citep{zhou2025rescriber} uses a small LLM for on-device chat privacy. Related work also studies entity-tagging approaches~\citep{mancera2025pba}, clinical-text anonymization~\citep{pissarra2024unlocking}, differential privacy for prompts~\citep{duan2023flocks}, and private-attribute randomization~\citep{frikha2025incognitext}. 
These methods improve privacy protection beyond span detection, but adversarial rewriting loops often treat attacker confidence as direct evidence of residual leakage. This can trigger unnecessary rewriting when a guess is plausible under world knowledge but not actually supported by the original text.

\paragraph{Dual-system inference and uncertainty gating.}
Routing only hard cases to an expensive model is a well-established efficiency strategy. FrugalGPT~\citep{chen2023frugalgpt} cascades LLMs to reduce inference cost; BranchyNet~\citep{teerapittayanon2016branchynet} adds early exits for neural networks; and fast/slow hybrid architectures draw on the broader dual-system view of decision-making~\citep{kahneman2011fast}. Evidential deep learning~\citep{sensoy2018evidential} provides a principled mechanism to estimate epistemic uncertainty and decide when to escalate. In privacy mediation, however, uncertainty gating cannot be only cost-driven: fast-path decisions must remain grounded in policy evidence, and uncertain cases must be escalated conservatively, as errors may expose sensitive information.

\paragraph{Positioning.}
\sys{} connects these lines of work by formulating live privacy mediation as policy-conditioned risk-adaptive anonymization. A structured risk controller links policy evidence and action context to both action-level decisions and anonymization strength, while verified adversarial feedback ensures that rewriting escalates only in response to residual leakage supported by the original text rather than hallucinated attacker confidence.

% =====================================================================
\section{Method}
\label{sec:method}

\begin{figure*}[t]
\vspace{-0.8em}
    \centering
    \includegraphics[width=\linewidth]{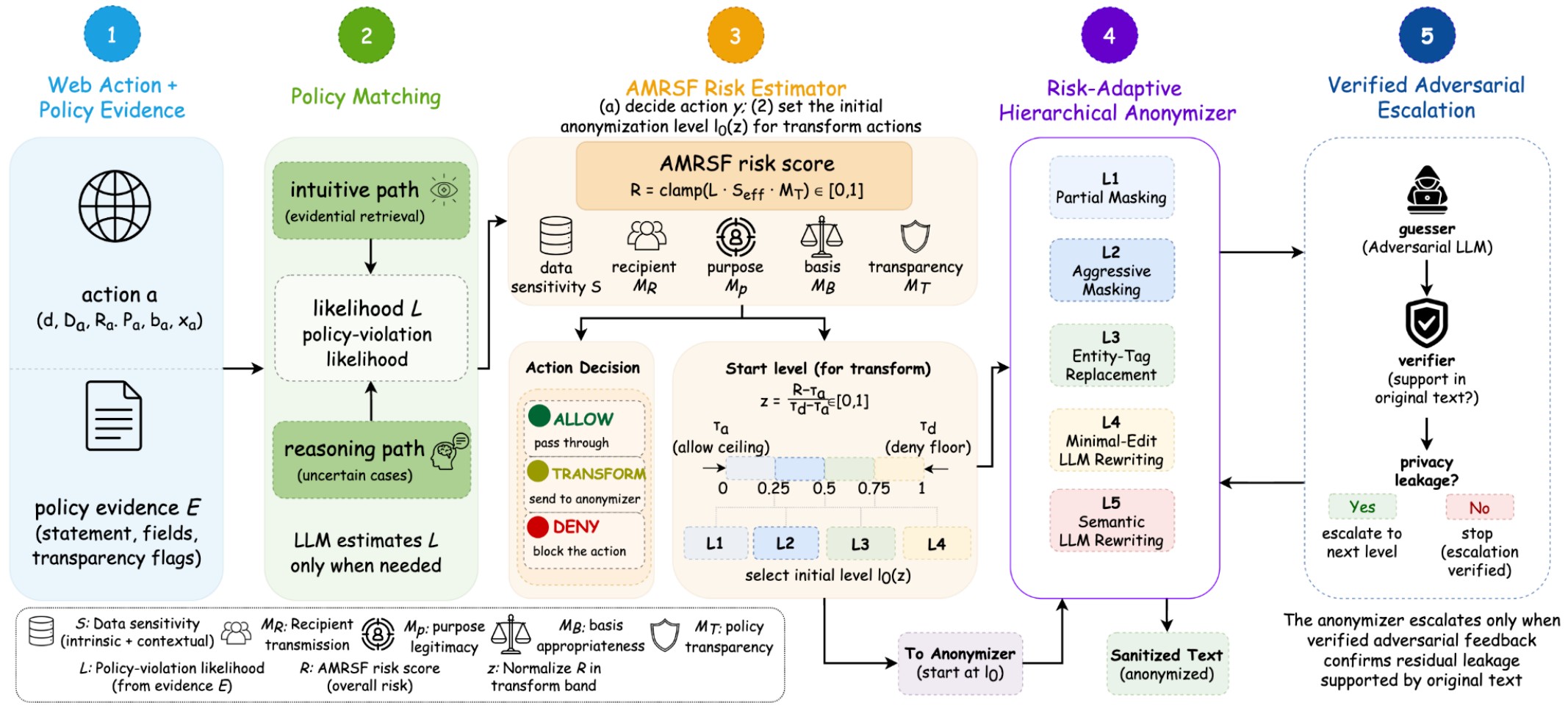}
    \caption{
    Framework of \sys{}. Policy matching estimates the violation likelihood \(L\), and \amrsf{} converts it into an action-conditioned risk score \(R\). The score has two effects: it determines the action decision \(y\in\{\texttt{allow},\texttt{transform},\texttt{deny}\}\), and, for transformed actions, its normalized transform-band position \(z\) selects the initial anonymization level \(\ell_0(z)\). The hierarchical anonymizer then escalates when leakage is detected.
    }
    \label{fig:framework}
\vspace{-1em}
\end{figure*}

\subsection{Overview}
\label{sec:method-overview}

A web action is represented as
\[
\mathbf{a}=(d,\mathcal{D}_{\mathbf{a}},\mathcal{R}_{\mathbf{a}},
\mathcal{P}_{\mathbf{a}},b_{\mathbf{a}},x_{\mathbf{a}})\in\mathcal{A},
\]
where \(d\) is the destination domain, \(\mathcal{D}_{\mathbf{a}}\) the data categories (e.g., email, location, health, financial, or free-text content),
\(\mathcal{R}_{\mathbf{a}}\) the recipients, \(\mathcal{P}_{\mathbf{a}}\) the declared purposes,
\(b_{\mathbf{a}}\in\{\text{user-initiated},\text{default},\text{background}\}\) the action basis,
and \(x_{\mathbf{a}}\in\Sigma^\ast\) the outgoing text.  Here, \texttt{user-initiated}
denotes an explicit user action such as submitting a form or pasting text, \texttt{default}
denotes routine processing needed for the requested service, and \texttt{background}
denotes passive or script-triggered transmission without an immediate user action. Privacy policies are segmented into
atomic statements \(\mathcal{S}\), each annotated with structured fields such as data
category, recipient, purpose, legal basis, retention, and user-rights disclosures.  A retriever
\(\textsc{Top}_k\) selects the \(k\) most relevant policy statements as evidence
\(\mathcal{E}=\textsc{Top}_k(\mathbf{a},\mathcal{S})\).The matcher then emits an action decision and risk score,
\(
(y,R)=M(\mathbf{a},\mathcal{E}), \)
\(
\mathcal{Y}=\{\texttt{allow},\texttt{transform},\texttt{deny}\},
\)
and \sys{} implements
\begin{equation}
\Phi(\mathbf{a},\mathcal{S})
=
A\bigl(\mathbf{a},y,R\bigr)
\in \mathcal{Y}\times\Sigma^\ast .
\label{eq:system}
\end{equation}
If \(y=\texttt{transform}\), the anonymizer \(A\) uses the transform-band position of \(R\)
to select an initial rewrite level and then uses verified adversarial feedback to decide
whether further escalation is needed. \sys{} separates two controls: \amrsf{} determines the starting strength from action
context, while the verified guesser adjusts that strength only when residual leakage is
surface-supported. The overall framework is introduced in Figure~\ref{fig:framework}.

\subsection{Efficient Policy Matching}
\label{sec:efficient-matching}

The policy matcher estimates the likelihood that an outgoing action conflicts with the
retrieved policy evidence. The likelihood \(L\) is
then used as a learned input to the downstream risk controller \(R\).
To estimate \(L\) under the latency constraints of live web traffic, \sys{} uses a shared
evidence retriever followed by a two-path matcher. Pre-segmented policies \(s\in\mathcal{S}\) come with structured fields---data categories, actions, purposes,
recipients, legal basis, rights, and retention---together with the original free-text snippet.

The evidence set \(\mathcal{E}=\textsc{Top}_k(\mathbf{a},\mathcal{S})\) is retrieved using
a hybrid score combining structured field overlap with lightweight lexical matching 
\begin{equation}
\textsc{Score}(\mathbf{a},s)
=
\alpha\,\tilde{\textsc{S}}(\mathbf{a},s)
+
(1-\alpha)\,\tilde{\textsc{K}}(\mathbf{a},s),
\label{eq:hybrid}
\end{equation}
where \(\tilde{\textsc{S}}\) is the normalized structured score,
\(\tilde{\textsc{K}}\) is a normalized BM25Lite score, \(\alpha=0.6\), and \(k=8\).

The fast path applies an evidential classifier to the action--evidence pair and outputs
both the risky-class mass and epistemic uncertainty \(u\). When
\begin{equation}
u < \tau_u,
\label{eq:gate}
\end{equation}
the risky-class mass is used as \(L\). Otherwise, an LLM reasoning path estimates \(L\) from
the same action and evidence.  Importantly, the LLM estimates only \(L\): it does
not directly determine the final risk score, action decision, or anonymization level.
The encoder, retrieval weights, prompts, and online fast-path update are specified in
Appendix~\ref{app:retrieval-params} and Appendix~\ref{app:prompts}.

\subsection{Policy-Conditioned Risk Estimator}
\label{sec:policy-matcher}

The risk estimator converts the policy-violation likelihood \(L\) into an operational
decision by accounting for the sensitivity and transmission context of the outgoing action.
Given action \(\mathbf{a}\), retrieved evidence \(\mathcal{E}\), and the likelihood \(L\)
estimated by the matcher, \amrsf{} computes the action-conditioned risk score
\begin{equation}
R(\mathbf{a},\mathcal{E}) =
\mathrm{clamp}_{[0,1]}\!\bigl(L \cdot S_{\mathrm{eff}} \cdot M_T\bigr),
\label{eq:amrsf}
\end{equation}
where
\begin{equation}
S_{\mathrm{eff}} =
\min\!\bigl(1,\;S_d\cdot M_{tr}\cdot M_p\cdot M_b\bigr).
\label{eq:seff}
\end{equation}
Here \(S_d\) is base data sensitivity; \(M_{tr}\), \(M_p\), and \(M_b\) adjust severity
according to recipient transmission, purpose legitimacy, and action basis; and \(M_T\) accounts for
policy transparency. Thus, unlike an
LLM-assigned scalar risk judgment, \(R\) explicitly separates violation likelihood from
the contextual severity of the outgoing action. These fixed factors are grounded in regulatory and empirical privacy-risk
studies
\citep{mccallister2010guide,regulation2016regulation,nissenbaum2009privacy,
milne2004strategies,ackerman1999privacy,bhatia2018empirical}.
Full numerical factors are reported in Tables~\ref{tab:data-full}--\ref{tab:transparency}
of Appendix~\ref{app:amrsf}.

Let
\(t(\mathbf{a})\in\{\text{crit},\text{high},\text{mod},\text{low}\}\) denote the highest
data-sensitivity tier in \(\mathcal{D}_{\mathbf{a}}\), with allow ceiling \(\tau_a^t\)
and deny floor \(\tau_d^t\). Then action decision
\begin{equation}
y =
\begin{cases}
\texttt{allow} & R < \tau_a^{t(\mathbf{a})},\\
\texttt{transform} & \tau_a^{t(\mathbf{a})} \le R < \tau_d^{t(\mathbf{a})},\\
\texttt{deny} & R \ge \tau_d^{t(\mathbf{a})}.
\end{cases}
\label{eq:decision}
\end{equation}
Sensitive tiers use a wider \texttt{transform} region so that uncertain high-risk actions
are rewritten rather than released unchanged. Threshold values are reported in
Table~\ref{tab:decision-thresholds} of Appendix~\ref{app:amrsf}.

\subsection{Risk-Adaptive Hierarchical Anonymizer}
\label{sec:anonymizer}

\begin{figure*}[t]
\vspace{-0.8em}
    \centering
    \includegraphics[width=\linewidth]{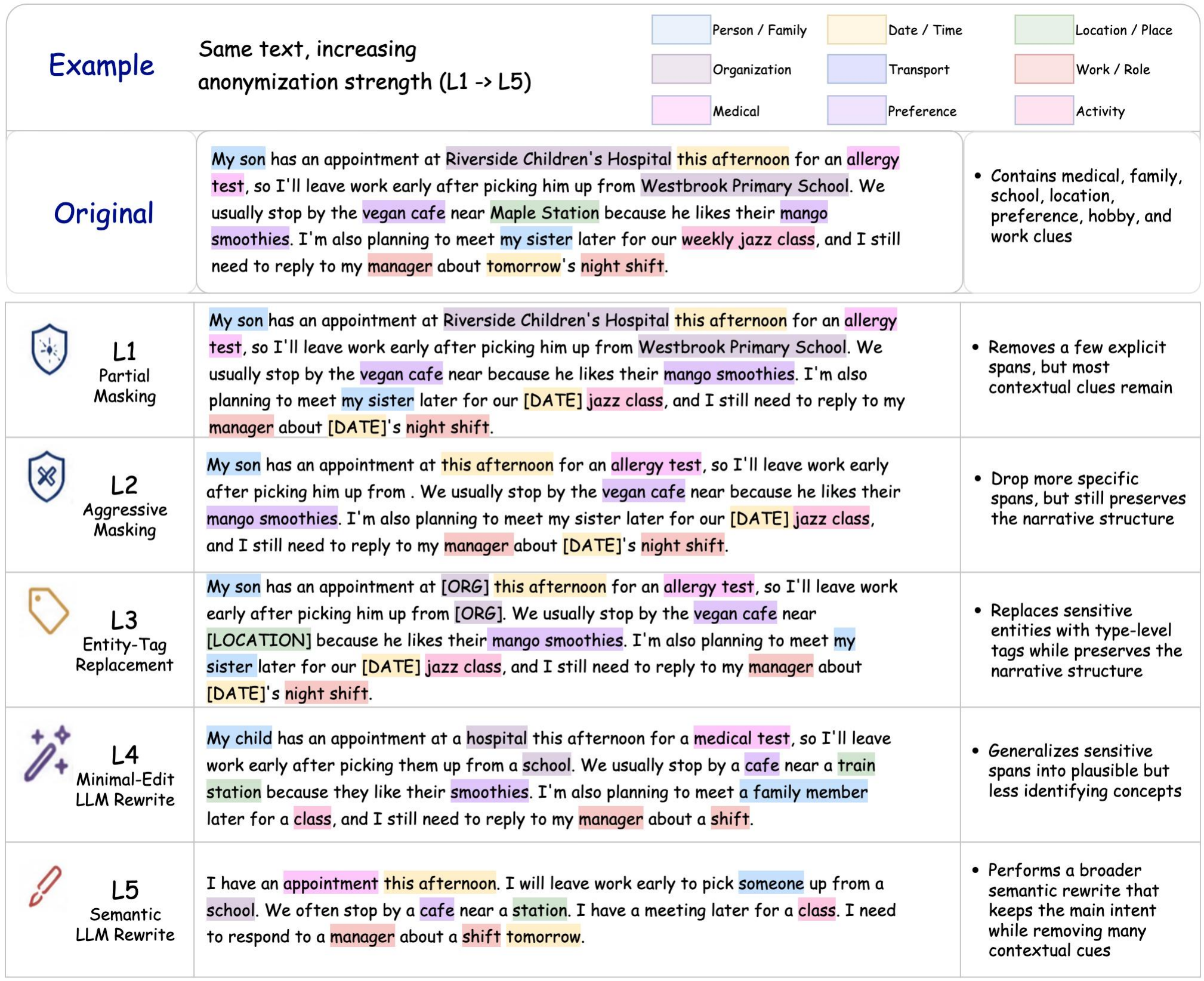}
    \caption{
    Example of the five-level anonymization hierarchy on a contextually sensitive paragraph. From L1 to L5, more entities and contextual cues are anonymized, making the text less specific while preserving overall meaning.
    }
    \label{fig:level-example}
\vspace{-0.8em}
\end{figure*}

The hierarchical anonymizer controls how strongly transformed content is rewritten, using risk to select
a starting strength and verified feedback to determine whether further escalation is needed. When \(y=\texttt{transform}\), \sys{} rewrites the outgoing text \(x_{\mathbf{a}}\) using a five-level risk-adaptive hierarchy: L1 partial masking,
L2 aggressive masking, L3 entity-tag replacement, L4 minimal-edit LLM rewriting, and L5
semantic LLM rewriting. The initial level is determined by its normalized position within
the transform band:
\begin{equation}
z(\mathbf{a}) =
\operatorname{clip}_{[0,1]}\!\left(
\frac{R - \tau_a^{t(\mathbf{a})}}
     {\tau_d^{t(\mathbf{a})} - \tau_a^{t(\mathbf{a})}}
\right),
\label{eq:band-z}
\end{equation}
and select the initial anonymization level by \(z\):
\vspace{-0.8em}
\begin{equation}
\ell_0(z) =
\begin{cases}
1 & z<0.25,\\
2 & 0.25\le z<0.50,\\
3 & 0.50\le z<0.75,\\
4 & z\ge 0.75.
\end{cases}
\label{eq:band-level}
\end{equation}
Thus, transformed actions near the allow ceiling start with lighter rewriting, while those
near the deny floor start from stronger levels. L1 partial masking, L2 aggressive masking, L3 entity-tag replacement, L4 minimal-edit LLM rewriting, and L5 semantic LLM rewriting. L1–L3 follow rule- and tag-based anonymization, L4 follows the minimal-edit paradigm, and L5 follows semantic iterative rewriting. Figure~\ref{fig:level-example} illustrates how the hierarchy progressively removes specific
entities and contextual cues while preserving the overall communicative intent. Lower levels preserve more of the original surface form and contextual detail, whereas higher levels are
reserved for cases where weaker rewrites leave residual leakage.

\subsection{Verified Adversarial Escalation}
\label{sec:guesser}

Prior adversarial anonymizers use an auxiliary attacker to identify residual personal
information and guide further rewriting~\citep{staab2025language,shao2025agentstealth}.
However, an attacker may assign high confidence to a plausible guess that is unsupported
by the original text, causing unnecessary escalation and utility loss.
\sys{} therefore verifies attacker guesses before allowing them to trigger stronger rewriting.
 Given anonymized text \(x_\ell\), the attacker returns guesses with confidences
\(\mathcal{G}(x_\ell)=\{(g,c_g)\}\). The verifier compares each guess against the original
text \(x_0\), which remains local to the anonymizer and is never exposed to the attacker.
Let \(\mathcal{W}(s)\) denote the set of content words in \(s\) with length greater than two.
A guess is admitted only if it is explicitly contained in \(x_0\) or shares a content word
with it:

\vspace{-0.4em}
{\small
\begin{equation}
\textsc{V}(g,x_0)
=
\mathbf{1}[g\subseteq x_0]
\vee
\mathbf{1}[\mathcal{W}(g)\cap\mathcal{W}(x_0)\neq\emptyset].
\label{eq:verify}
\end{equation}
}
\vspace{-0.7em}

The verified confidence is the maximum confidence among verified guesses:

\vspace{-0.4em}
{\small
\begin{equation}
\hat{c}(x_\ell,x_0)
=
\max\{c_g:(g,c_g)\in\mathcal{G}(x_\ell),\ \textsc{V}(g,x_0)=1\}.
\label{eq:vhat}
\end{equation}
}
If no guess is verified, \(\hat{c}=0\).
Starting from the initial level \(\ell_0(z)\), \sys{} anonymizes at level \(\ell\), queries the attacker, and
escalates to \(\ell+1\) only if \(\hat{c}\ge\tau_g\). The loop terminates when
\(\hat{c}<\tau_g\), when \(\ell=5\), or after \(T=5\) rounds. Verification filters hallucinated or prior-driven guesses while preserving escalation signals for leakage that is supported by the
original text. Equation~\ref{eq:verify} is explicitly lexical: it does not recognize a correct contextual inference that is surface-disjoint from $x_0$. We therefore treat it as an anti-hallucination and escalation-cost control rather than as the system's contextual-inference defence; higher-risk flows instead start at L4 through $R$ and receive semantic rewriting without waiting for the verifier. The pseudocode, lexical verification guarantee, and the expected-cost analysis are in
Appendix~\ref{app:analysis}.

% =====================================================================
\section{Evaluation}
\label{sec:eval}

\begin{table*}[!b]
\centering
\footnotesize
\setlength{\tabcolsep}{2pt}
\renewcommand{\arraystretch}{0.92}
\resizebox{\textwidth}{!}{
\begin{tabular}{l c c c c | c c c c | c c c c}
\toprule
\rowcolor{black!10}
& \multicolumn{4}{c|}{\textbf{TAB (legal, $n=811$)}}
& \multicolumn{4}{c|}{\textbf{SynthPAI (Reddit, $n=200$)}}
& \multicolumn{4}{c}{\textbf{PII-Masking-300k ($n=500$)}} \\
\rowcolor{black!10}
\textbf{Method}
& Priv$\uparrow$ & Util$\uparrow$ & Guess$\downarrow$ & Lat.
& Priv$\uparrow$ & Util$\uparrow$ & Guess$\downarrow$ & Lat.
& Priv$\uparrow$ & Util$\uparrow$ & Guess$\downarrow$ & Lat. \\
\midrule
\multicolumn{13}{l}{\textit{Iterative / multi-pass LLM rewriters}} \\
\textbf{A. \ours{} (full, L1--L5)}
& \best{$\mathbf{.945}_{.010}$} & $.696_{.008}$ & \best{$\mathbf{.283}_{.019}$} & \best{$\mathbf{5052}_{181}$}
& \second{$\underline{.944}_{.007}$} & \best{$\mathbf{.888}_{.006}$} & $.295_{.009}$ & \best{$9949_{467}$}
& \best{$\mathbf{.901}_{.011}$} & $.488_{.008}$ & \best{$.213_{.010}$} & \best{$7453_{299}$} \\

B. Staab'25
& $.695_{.013}$ & $.564_{.017}$ & \second{$.335_{.019}$} & $25152_{239}$
& \best{$\mathbf{.975}_{.005}^{\dagger}$} & $.257_{.018}$ & \best{$\mathbf{.098}_{.014}$} & $45418_{5663}$
& \second{$\underline{.850}_{.010}$} & $.384_{.021}$ & \second{$\underline{.275}_{.020}$} & $24975_{96}$ \\

C. AgentStealth'25
& $.689_{.012}$ & \best{$\mathbf{.789}_{.007}$} & $.579_{.044}$ & $16928_{563}$
& $.864_{.009}$ & $.700_{.005}$ & \second{$.260_{.009}$} & $15485_{467}$
& $.680_{.008}$ & \second{$.634_{.015}$} & $.508_{.015}$ & $19414_{394}$ \\

D. HaS'23
& \second{$\underline{.926}_{.005}$} & $.669_{.009}$ & $.336_{.032}$ & \second{$5328_{594}$}
& $.811_{.026}$ & $.841_{.005}$ & $.477_{.020}$ & $12892_{755}$
& $.809_{.010}$ & $.586_{.011}$ & $.322_{.015}$ & \second{$\underline{8384}_{285}$} \\

E. IncogniText'25
& $.178_{.023}$ & \second{$\underline{.765}_{.011}$} & $.538_{.025}$ & $8341_{325}$
& $.553_{.019}$ & \second{$.855_{.013}$} & $.281_{.023}$ & \second{$10092_{413}$}
& $.315_{.011}$ & \best{$\mathbf{.701}_{.003}$} & $.624_{.018}$ & $9819_{202}$ \\

\midrule
\multicolumn{13}{l}{\textit{Single-shot LLM rewriters}} \\
F. CONFAIDE'24
& $\mathbf{.913}_{.005}$ & $\underline{.638}_{.008}$ & $\underline{.307}_{.027}$ & $\mathbf{679}_{97}$
& $\mathbf{.842}_{.013}$ & $\underline{.810}_{.010}$ & $\underline{.254}_{.015}$ & $\mathbf{2955}_{309}$
& $.776_{.009}$ & $\underline{.533}_{.004}$ & $.331_{.012}$ & $\mathbf{2449}_{182}$ \\

G. Pissarra'24
& $\underline{.903}_{.010}$ & $\mathbf{.701}_{.005}$ & $\mathbf{.257}_{.030}$ & $\underline{789}_{127}$
& $.815_{.002}$ & $\mathbf{.873}_{.003}$ & $.395_{.006}$ & $\underline{2943}_{293}$
& $\underline{.864}_{.010}$ & $\mathbf{.548}_{.008}$ & $\underline{.256}_{.013}$ & $2846_{157}$ \\

H. Rescriber'25
& $.719_{.022}$ & $.605_{.031}$ & $.393_{.022}$ & $3774_{398}$
& $\underline{.819}_{.012}$ & $.432_{.020}$ & $\mathbf{.182}_{.012}$ & $3241_{345}$
& $\mathbf{.892}_{.009}$ & $.282_{.008}$ & $\mathbf{.178}_{.013}$ & $\underline{2752}_{76}$ \\

I. DP-Prompt'23
& $.375_{.029}$ & $.478_{.015}$ & $.662_{.019}$ & $837_{146}$
& $.486_{.019}$ & $.539_{.032}$ & $.319_{.017}$ & $3371_{429}$
& $.524_{.025}$ & $.527_{.008}$ & $.505_{.012}$ & $3611_{97}$ \\

\midrule
\multicolumn{13}{l}{\textit{Tag-based baselines}} \\
J. Presidio
& $.703_{.015}$ & $.783_{.008}$ & $.216_{.023}$ & $14_{1}$
& $.659_{.005}$ & $.944_{.001}$ & $.375_{.015}$ & $28_{1}$
& $.511_{.019}$ & $.750_{.011}$ & $.489_{.022}$ & $24_{1}$ \\

K. PBa-LLM'25
& $.455_{.033}$ & $.809_{.009}$ & $.465_{.019}$ & $81_{13}$
& $.635_{.002}$ & $.956_{.000}$ & $.396_{.008}$ & $170_{33}$
& $.184_{.013}$ & $.906_{.005}$ & $.674_{.013}$ & $91_{1}$ \\
\bottomrule
\end{tabular}
}
\caption{
Main comparison across TAB, SynthPAI, and PII-Masking-300k.
All three corpora report 4-seed mean$_{\text{std}}$ (42--45).
Latency is mean ms per sample.
\textbf{Best} and \underline{second-best} are marked within each method block and dataset.
$^{\dagger}$Staab's high SynthPAI privacy is inflated by empty or placeholder outputs, which also explains its low utility.
}
\label{tab:main}
\vspace{-0.8em}
\end{table*}

The evaluation addresses four questions: 
(Q1) Does \sys{} improve the privacy--utility trade-off? 
(Q2) Does action-conditioned risk control both mediation decisions and anonymization strength?
(Q3) Does verified hierarchical anonymization outperform fixed rewriting strategies?
(Q4) Are the risk controller and intuitive+reasoning path calibrated and efficient?
Unless otherwise stated, all main results use the final \sys{} variant with the full L1--L5 hierarchy. Calibration results, additional confidence intervals, rewrite strategy comparison, seed variance, and robustness checks are reported in Appendix~\ref{app:calibration},~~\ref{app:rewrite} and~\ref{app:robust}.

\subsection{Setup}
\label{sec:eval-setup}

\paragraph{Benchmarks.}

Our evaluation combines anonymization, policy-matching, and risk-calibration data. For anonymization, we use three benchmarks \textbf{TAB}~\citep{pilan2022text}, which contain legal texts from ECHR judgements; \textbf{SynthPAI}~\citep{yukhymenko2024synthetic}, which
contains Reddit-style posts with personal attributes; and \textbf{PII-Masking-300k}~\citep{ai4privacy_2024}, which
contains XML, JSON, tabular, and free-form synthetic documents annotated with explicit PII
categories. For policy matching, we use \textbf{OPP-115}~\citep{wilson2016creation}. For \amrsf{} calibration, we use
an expert-rated action-scenario set and the public corpus from Staab.

\paragraph{Backbone and baselines.}
The main LLM backbone is Llama-3.2-3B-Instruct~\citep{grattafiori2024llama}; System~2, L4 rewriting, and the adversarial guesser use
the same backbone to ensure a controlled comparison across components. We compare against baselines spanning adversarial rewriting, prompt protection, contextual-integrity
prompting, on-device rewriting, differential privacy, private-attribute randomization, clinical
anonymization, and tag-based anonymization: Staab~\citep{staab2025language},
AgentStealth~\citep{shao2025agentstealth}, HaS~\citep{chen2023hide},
CONFAIDE~\citep{mireshghallah2023can}, Rescriber~\citep{zhou2025rescriber},
DP-Prompt~\citep{duan2023flocks}, IncogniText~\citep{frikha2025incognitext},
Pissarra~\textit{et al.}~\citep{pissarra2024unlocking}, and
PBa-LLM~\citep{mancera2025pba}.

\paragraph{Metrics.}
\emph{Privacy} is one minus the fraction of ground-truth sensitive spans retained in the output. \emph{Utility}
is token-Jaccard on TAB and PII-Masking-300k, and BERTScore F1 on SynthPAI. \emph{Guess} is the maximum verified adversarial-guesser confidence
on the output, and \emph{latency} is the wall-clock time per sample.

\begin{figure*}[!t]
    \centering
    \includegraphics[width=\linewidth]{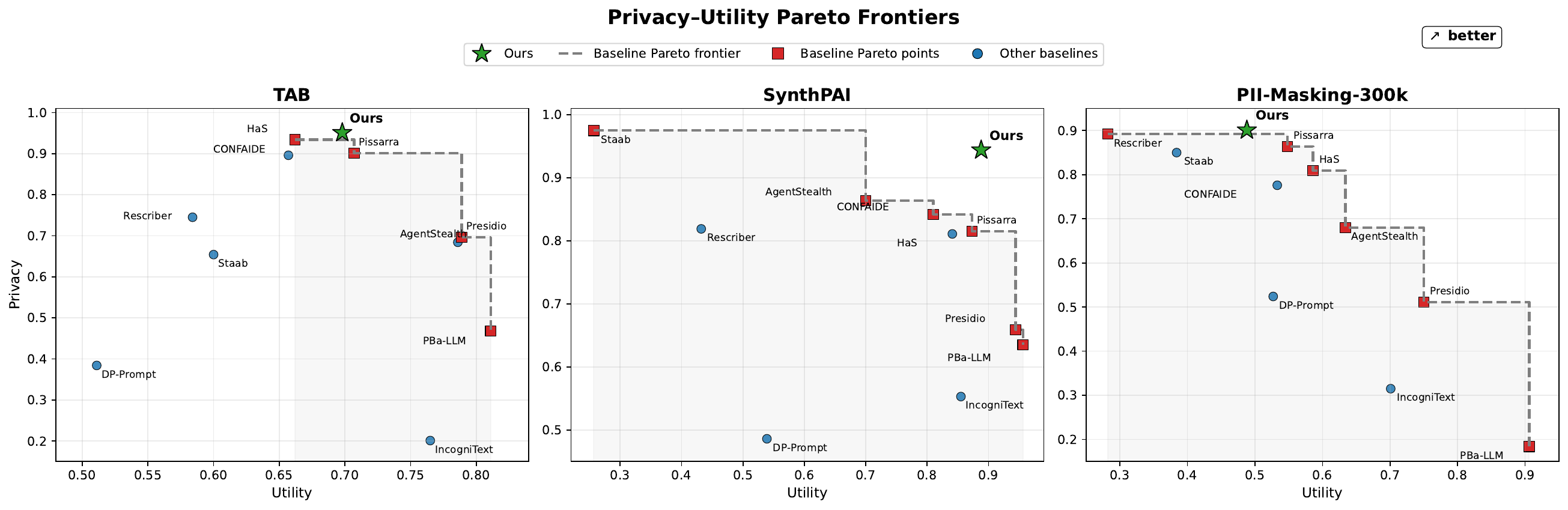}
    \caption{
    Privacy--utility Pareto frontiers across the three benchmarks.
    Baseline methods define each frontier, with \sys{} plotted separately.
    \sys{} lies above or beyond the baseline frontier on all three datasets, indicating a stronger cross-domain privacy--utility trade-off.} 
    \label{fig:pareto}
    \vspace{-0.8em}
\end{figure*}

\begin{table*}[!b]
\centering\footnotesize\setlength{\tabcolsep}{3pt}
\begin{tabular}{l c c c c c}
\toprule
\rowcolor{black!10}\textbf{Context} & $R$ & Allow & Trans. & Deny & $\overline{\ell_\text{cf}}$ \\
\midrule
$C_\text{low}$\,(first-party / functionality / paste / full)
& $.119_{.017}$ & $\mathbf{100.0}$ & $0.0$ & $0.0$ & $2.02_{1.31}$ \\
$C_\text{med}$\,(provider / analytics / default / partial)
& $.413_{.057}$ & $0.5$ & $\mathbf{98.0}$ & $1.5$ & $2.10_{1.29}$ \\
$C_\text{high}$\,(advertising / profiling / background / vague)
& $.909_{.016}$ & $0.0$ & $0.5$ & $\mathbf{99.5}$ & $4.33_{0.50}$ \\
$C_\text{ext}$\,(data broker / advertising / background / missing)
& $.909_{.011}$ & $0.0$ & $0.0$ & $\mathbf{100.0}$ & $4.34_{0.49}$ \\
\bottomrule
\end{tabular}
\caption{
Action-conditioned decision control on 200 PII-Masking samples with the same texts and fixed \(L=0.7\).
Riskier action contexts shift decisions from \texttt{allow} to \texttt{transform} and \texttt{deny}.
\(\bar{\ell}_{cf}\) is a counterfactual final anonymization level computed by routing all samples through the anonymizer; at deployment time, \texttt{allow} samples are released unchanged, while \texttt{deny} samples are blocked.
}
\label{tab:risk-control-decision}
\end{table*}

\subsection{Main Comparison}
\label{sec:eval-main}

Table~\ref{tab:main} compares \sys{} with published baselines across TAB, SynthPAI, and PII-Masking-300k. \sys{} delivers the most consistent, privacy--utility trade-off across domains. Figure~\ref{fig:pareto} reinforces this result: \sys{} lies above or beyond the baseline privacy--utility frontier on all three datasets, and is the only method above $0.90$ privacy on all domains and has the smallest
cross-domain spread. This consistency suggests that verified, risk-adaptive escalation reduces
domain-specific failure modes, especially when moving across legal prose, Reddit-style
personal attributes, and structured PII records.

Table~\ref{tab:risk-control-decision} evaluates the decision side of the controller by scoring the same 200 PII-Masking texts under four action contexts with \(L=0.7\) fixed. As recipient, purpose, action basis, and policy transparency become riskier, mean risk increases and decision distribution shifts from \texttt{allow} through \texttt{transform} to \texttt{deny}. This shows that \amrsf{} conditions mediation decisions on action context rather than on text content alone.

\subsection{Transform-Band Risk-to-Level Control}
\label{sec:eval-risk-control}

To isolate anonymization strength, the same 200 PII-Masking texts are
evaluated under four contexts constrained to the \texttt{transform} band.
For each action, \(z\) denotes the normalized position of \(R\) between
the allow ceiling and deny floor (Eq.~\ref{eq:band-z}), and
\(\ell_0(z)\) is selected by quartile (Eq.~\ref{eq:band-level}).

\begin{table*}[!tbp]
\centering\footnotesize\setlength{\tabcolsep}{4pt}
\begin{tabular}{l l l l c c c c c c c c c}
\toprule
\rowcolor{black!10}\textbf{Context} & Recipient & Purpose & Basis & $R$ & $z$ & Trans. \% & $\overline{\ell_0(z)}$ & L1 & L2 & L3 & L4 & $\overline{\ell_\text{final}}$ \\
\midrule
$T_\text{mild}$       & first-party      & Personalization & default & $.340_{.050}$ & $.252_{.179}$ & $99.0$ & $1.48$ & $\mathbf{65}$ & $22$ & $13$ & $0$  & $2.37$ \\
$T_\text{moderate}$   & service-provider & Personalization & default & $.457_{.066}$ & $.538_{.208}$ & $89.0$ & $2.52$ & $1$ & $\mathbf{64}$ & $16$ & $18$ & $2.93$ \\
$T_\text{strong}$     & service-provider & Marketing       & default & $.557_{.082}$ & $.744_{.170}$ & $81.5$ & $3.33$ & $1$ & $0$ & $\mathbf{64}$ & $35$ & $3.64$ \\
$T_\text{near-deny}$  & advertising      & Marketing       & default & $.704_{.053}$ & $.980_{.080}$ & $13.0$ & $3.98$ & $0$ & $0$ & $0$ & $\mathbf{99}$ & $4.42$ \\
\bottomrule
\end{tabular}
\caption{
Transform-band risk-to-level control on 200 PII-Masking samples with fixed \(L=0.7\).
As normalized band position \(z\) increases, the dominant starting level shifts from L1 to L4, the mean initial level \(\overline{\ell_0(z)}\) rises monotonically, and the final level after verified escalation also increases.
\texttt{Trans. \%} is the fraction of samples remaining within the \texttt{transform} region.
}

\vspace{-0.8em}
\label{tab:risk-to-level-band}
\end{table*}

Table~\ref{tab:risk-to-level-band} shows that anonymization strength follows
the action-conditioned band position. As \(z\) increases, the dominant
starting level shifts from L1 to L4 and \(\overline{\ell_0(z)}\) rises
monotonically. Verified adversarial escalation preserves this ordering, with the mean final level rising from 2.37 to 4.42. Thus, among actions selected for
transformation, \sys{} does not use a fixed rewriting policy: \amrsf{} sets the starting strength from the action context, and verified feedback escalates only when residual leakage remains.

\subsection{Ablation}
\label{sec:eval-ablation}

\begin{table*}[!t]
\centering\footnotesize\setlength{\tabcolsep}{4pt}
\begin{tabular}{l c c c | c c c | c c c}
\toprule
\rowcolor{black!10} & \multicolumn{3}{c|}{\textbf{TAB}} & \multicolumn{3}{c|}{\textbf{SynthPAI}} & \multicolumn{3}{c}{\textbf{PII-Mask}} \\
\rowcolor{black!10}\textbf{Cap} & Priv & Util & Guess & Priv & Util & Guess & Priv & Util & Guess \\
\midrule
$\ell_\text{max}{=}3$ & $.924$ & $\mathbf{.709}$ & $.356$ & $.913$ & $.897$ & $.329$ & $.869$ & $\underline{.533}$ & $\underline{.247}$ \\
$\ell_\text{max}{=}4$ & $.933$ & $.706$ & $.325$ & $.912$ & $\mathbf{.905}$ & $.316$ & $.849$ & $\mathbf{.539}$ & $.283$ \\
\textbf{$\ell_\text{max}{=}5$} & $\mathbf{.945}_{.011}$ & $.696_{.009}$ & $\mathbf{.296}$ & $\mathbf{.944}_{.007}$ & $.888_{.006}$ & $\mathbf{.299}$ & $\mathbf{.901}_{.011}$ & $.488_{.008}$ & $\mathbf{.213}_{.010}$ \\
\bottomrule
\end{tabular}
\caption{
Level-progression ablation with maximum level $\ell_{\max}\in\{3,4,5\}$.
L3 uses only rule-based masking/tagging, L4 adds minimal-edit rewriting, and L5 adds the semantic safety-net.
% Rows with $\ell_{\max}=5$ report 4-seed mean$_{\text{std}}$; rows with $\ell_{\max}\in\{3,4\}$ are single-seed.
}
\label{tab:level-progression}
\vspace{-0.8em}
\end{table*}

\begin{table}[!htbp]
\centering\small\setlength{\tabcolsep}{4pt}
\begin{tabular}{l c c r}
\toprule
\rowcolor{black!10}\textbf{Variant} & Priv$\uparrow$ & Util$\uparrow$ & ms \\
\midrule
\multicolumn{4}{l}{\textit{SynthPAI}} \\
A. NER only & $.483_{.003}$ & $.944_{.001}$ & 23 \\
B. + guesser & $.569_{.007}$ & $.938_{.001}$ & 729 \\
C. + LLM rewrite & $.910_{.011}$ & $.899_{.003}$ & 2518 \\
\textbf{D. + L1--L5 adaptive} & $\mathbf{.944}_{.007}$ & $.888_{.006}$ & 9949 \\
\midrule
\multicolumn{4}{l}{\textit{TAB}} \\
A. NER only & $.848_{.022}$ & $.735_{.009}$ & 11 \\
B. + guesser & $.895_{.013}$ & $.718_{.007}$ & 2731 \\
C. + LLM rewrite & $.915_{.018}$ & $.725_{.010}$ & 831 \\
\textbf{D. + L1--L5 adaptive} & $\mathbf{.945}_{.011}$ & $.696_{.009}$ & 5052 \\
\midrule
\multicolumn{4}{l}{\textit{PII-Masking-300k}} \\
A. NER only & $.389_{.014}$ & $.813_{.009}$ & 14 \\
B. + LLM rewrite & $.853_{.007}$ & $.546_{.008}$ & 4084 \\
\textbf{C. + L1--L5 adaptive} & $\mathbf{.901}_{.011}$ & $.488_{.008}$ & 7453 \\
\bottomrule
\end{tabular}
\caption{
Component ablation on SynthPAI, TAB, and PII-Masking-300k (mean$_{\text{std}}$ across 4 seeds; $n=200$ for SynthPAI/TAB and $n=500$ for PII-Masking).
Variants are cumulative: NER only, plus adversarial feedback, plus single-level LLM rewriting, and the full verified L1--L5 adaptive stack.
}
\label{tab:ablation}
\vspace{-1em}
\end{table}

Table~\ref{tab:ablation} decomposes the contributions of the main components. NER alone is strong on TAB, where many sensitive spans are explicit legal entities, but it is less effective on SynthPAI, where personal attributes are more implicit and contextual. Adding adversarial feedback and LLM rewriting improves privacy, while the full L1--L5 adaptive stack achieves the best privacy across both datasets at a moderate utility cost. Additional comparisons between semantic-only and minimal-edit-only rewriting are reported in Table~\ref{tab:l4-pareto} (Appendix~\ref{app:robust}).
Table~\ref{tab:level-progression} isolates the role of the anonymization hierarchy. L1--L3 remove many explicit spans, especially in legal prose and structured records. L4 adds high-utility minimal edits, while L5 is most useful when weaker rewrites leave contextual leakage. The results on PII-Masking-300k shows why a fixed maximum level can be brittle: minimal edits may preserve too much structure, whereas the semantic safety-net recovers privacy in harder cases.
Qualitative side-by-side rewrites are in Appendix~\ref{app:casestudies}.

\subsection{Calibration, Efficiency, and Robustness}
\label{sec:eval-robust}

\amrsf{} is evaluated on an expert-rated scenario set and a larger Staab-derived calibration corpus. The score aligns well with annotated risk in both settings, while the larger corpus provides a more stable estimate because the expert-rated set is small. The evidential fast path is evaluated on OPP-115 policy-action pairs and escalates only a small fraction of events to the LLM, supporting its role as an efficiency mechanism for live traffic mediation. Full calibration statistics, confidence intervals, and System~1 metrics are reported in Appendix~\ref{app:calibration}.

Robustness checks support the two adaptive mechanisms. Disabling verification leads to more level upgrades without improving privacy, while reducing utility and increasing latency, confirming that raw attacker confidence often reflects unsupported guesses. By switching backbone from Llama-3.2-3B to Qwen3-8B, \sys{} remains the strongest privacy--utility method on test benchmarks. Additional paired-bootstrap tests, seed variance, attacker-replacement experiments, parsing-failure analysis, backbone-portability details, and online fast-path self-improvement results are in Appendix~\ref{app:robust}.

% =====================================================================
\section{Conclusion}
\label{sec:conclusion}

This paper presented \sys{}, a policy-conditioned framework for risk-adaptive anonymization of live web traffic. Rather than relying on an LLM to directly assign privacy risk, \sys{} uses \amrsf{} to combine policy evidence and action context, determining both the allow/transform/deny decision and the initial anonymization level. A verified adversarial guesser then escalates rewriting only when attacker guesses are supported by the original text. Across legal text, Reddit-style personal attributes, and multi-format PII records, \sys{} improved the privacy--utility trade-off over published baselines, remained robust under a backbone switch, and adapted its rewriting strength across different action contexts.

\section*{Limitations}

The evaluation is limited to English and uses public or synthetic corpora; real browser traces, clinical and dialogue data, and multilingual settings require further study. \amrsf{} should also be viewed as a mediation controller rather than a legal compliance engine. Deployments should therefore prioritize local processing, disclose any cloud escalation, allow user override, and communicate that anonymization reduces but does not eliminate the re-identification risk.

\section*{Ethical considerations}
This work collects no new personal data and involves no human subjects. Experiments use
publicly released or synthetic benchmark datasets and an open-weight backbone or publicly
accessible model APIs; no private browser traces are used. GuardianAgent is intended to
reduce unintended disclosure, not to guarantee anonymity or legal compliance. Any
deployment should disclose rewriting and cloud escalation, preserve user control over
mediation decisions, and communicate residual re-identification and utility-loss risks.

\section*{Acknowledgments}

The authors thank the UNSW-UTS Trustworthy Digital Society for supporting this research. Also, this research
used the NVIDIA H100 cluster provided by SHARON AI.

% =====================================================================
\bibliography{custom}

\appendix

% =====================================================================
% =====================================================================
\section{System\,1 Architecture and Retrieval Hyperparameters}
\label{app:retrieval-params}

\paragraph{System\,1 encoder and EDL classifier (referenced from \S\ref{sec:efficient-matching}).} Behaviour and policy encoders $\phi_b,\phi_p:\mathcal{A}\!\cup\!\mathcal{S}\!\to\!\mathbb{R}^d$ are either a $d\!=\!42$ multi-hot variant over OPP-115 categories or a $d\!=\!384$ frozen sentence-transformer~\citep{reimers2019sentence}. Writing $b\!=\!\phi_b(\mathbf{a})$ and $p\!=\!\phi_p(\mathbf{s})$, the interaction feature map is
\begin{equation}
\psi(\mathbf{a},\mathbf{s}) = \bigl[\,b\,;\,p\,;\,b\!\odot\!p\,;\,|b\!-\!p|\,\bigr] \in \mathbb{R}^{4d}.
\end{equation}
A two-layer MLP with ReLU and dropout, $f_\theta:\mathbb{R}^{4d}\!\to\!\mathbb{R}^K$ with $K\!=\!2$, yields logits that we convert to Dirichlet parameters following EDL~\citep{sensoy2018evidential}: $\boldsymbol\alpha\!=\!\text{softplus}(f_\theta(\psi))\!+\!\mathbf{1}$, $S\!=\!\sum_k\alpha_k$. Class probabilities are $p_k\!=\!\alpha_k/S$ and epistemic uncertainty is $u\!=\!K/S$. The likelihood used in Eq.~\ref{eq:amrsf} is $L\!=\!\alpha_2/S$.

\paragraph{Retrieval scoring (referenced from \S\ref{sec:efficient-matching}).} The structured score over fields $\mathcal{F}\!=\!\{\text{data},\text{act},\text{purp},\text{recip}\}$ is
\begin{equation}
\textsc{S}(\mathbf{a},s) = \!\!\!\sum_{f\in\mathcal{F}}\!\!w_f\rho_f + \textsc{Flags}(s),
\end{equation}
with weights $(0.35,0.35,0.20,0.10)$, $\rho_f=\min(1, |\mathbf{a}_f\cap s_f|/\max(1,|\mathbf{a}_f|))$, and \textsc{Flags}$(s)$ adding small bonuses for disclosed user rights, trans-EEA transfer, and retention. BM25Lite uses $k_1\!=\!1.5$, $b\!=\!0.75$ over the tokenized snippet, with on-the-fly IDF (no global inverted index per-domain).

% =====================================================================
\section{Prompts}
\label{app:prompts}

\paragraph{System\,2 decider (system prompt).}
\begin{quote}\small\ttfamily
Given an outgoing action and retrieved policy evidence, estimate the likelihood that the
action conflicts with the policy. Return STRICT JSON:
\{violation\_likelihood: [0..1], rationale, evidence\_ids\}.
Do not decide whether the action should be allowed, transformed, or denied.
\end{quote}

\paragraph{Adversarial guesser (system prompt).}
\begin{quote}\small\ttfamily
You are an adversarial privacy auditor. Given anonymized text, try to reconstruct or guess the redacted information. For each guess, provide: field, guess, confidence (0.0--1.0), reasoning. Return JSON: \{``guesses'': [...], ``max\_confidence'': 0.0--1.0\}.
\end{quote}

\paragraph{L4 minimal-edit rewriter.}
\begin{quote}\small\ttfamily
Rewrite the following sentence so that the listed sensitive fields cannot be recovered, while preserving the communicative intent. Do not insert [TAGS]; produce natural text. Sensitive fields: \{fields\}. Input: ``\{text\}''.
\end{quote}

\paragraph{L5 semantic rewriter.}
\begin{quote}\small\ttfamily
Rewrite the text to preserve the communicative intent while removing identifying details and contextual cues. Full paraphrasing is allowed.
\end{quote}

% =====================================================================
\section{Full AMRSF Factor Tables}
\label{app:amrsf}

Table~\ref{tab:data-full} gives the full per-category data weights used to compute $S_d$ in Eq.~\ref{eq:seff}. Table~\ref{tab:transmission} gives the per-recipient transmission risk $M_{tr}$ and Table~\ref{tab:purpose} the per-purpose legitimacy $M_p$. Table~\ref{tab:transparency} lists the six-signal credits used by the transparency multiplier $M_T$.

\begin{table}[h]
\centering\small\setlength{\tabcolsep}{3pt}
\begin{tabular}{l c l}
\toprule
\rowcolor{black!10}\textbf{Category} & \textbf{Wt.} & \textbf{Source} \\
\midrule
Credentials & 1.00 & NIST~800-122 \\
Financial & 0.95 & Milne'04 \\
Health, Biometric, Genetic & 0.95 & GDPR Art.\,9 \\
SSN & 0.90 & NIST~800-122 \\
Location (precise) & 0.85 & Milne'04 \\
Location (unspec.) & 0.75 & default \\
Phone & 0.70 & Ackerman'99 \\
Content & 0.65 & Bhatia'18 \\
PostalAddress & 0.65 & Ackerman'99 \\
Contact & 0.60 & default \\
Email & 0.55 & Ackerman'99 \\
SearchHistory & 0.50 & Milne'04 \\
BrowsingHistory & 0.45 & Milne'04 \\
IPAddress, DeviceID & 0.40 & NIST \\
Location (coarse) & 0.35 & Milne'04 \\
AppUsage & 0.25 & low \\
Cookies & 0.20 & low \\
Language & 0.10 & low \\
\bottomrule
\end{tabular}
\caption{Full data-sensitivity weights \(\mathrm{DW}(c)\). Some rows group categories assigned the same weight.}
\label{tab:data-full}
\end{table}

\begin{table}[h]
\centering\small
\begin{tabular}{l r}
\toprule
\rowcolor{black!10}\textbf{Recipient} & \textbf{$M_{tr}$} \\
\midrule
first-party only & $0.20$ \\
service provider (Art.\,28) & $0.40$ \\
analytics / government & $0.50$ \\
social-media platform & $0.60$ \\
unknown third party & $0.70$ \\
advertising network & $0.80$ \\
data broker & $0.90$ \\
\bottomrule
\end{tabular}
\caption{Transmission risk $s_{tr}$ by recipient type.}
\label{tab:transmission}
\end{table}

\begin{table}[h]
\centering\small
\begin{tabular}{l r}
\toprule
\rowcolor{black!10}\textbf{Purpose} & \textbf{$M_p$} \\
\midrule
functionality, security, auth. & $0.15$ \\
personalization, research & $0.35$ \\
analytics & $0.45$ \\
marketing & $0.65$ \\
profiling (Art.\,22) & $0.75$ \\
advertising & $0.80$ \\
undisclosed & $0.65$ \\
\bottomrule
\end{tabular}
\caption{Purpose legitimacy $s_p$.}
\label{tab:purpose}
\end{table}

\begin{table}[h]
\centering\small
\begin{tabular}{l r}
\toprule
\rowcolor{black!10}\textbf{Signal} & \textbf{Credit} \\
\midrule
substantial text ($\ge 50$ chars) & $-0.10$ \\
covers the action's data category & $-0.15$ \\
covers the action's declared purpose & $-0.10$ \\
states a legal basis (GDPR Art.\,6) & $-0.08$ \\
discloses user rights & $-0.05$ \\
specifies retention & $-0.02$ \\
\midrule
$M_T = \max\!\bigl(0.8,\; 1.3 + \sum\text{credits}\bigr)$ & \\
\bottomrule
\end{tabular}
\caption{Transparency multiplier $M_T$ signal credits.}
\label{tab:transparency}
\end{table}

\begin{table}[h]
\centering\small
\setlength{\tabcolsep}{5pt}
\begin{tabular}{l c c}
\toprule
\rowcolor{black!10}
\textbf{Data tier} & \textbf{Allow ceiling \(\tau_a^t\)} & \textbf{Deny floor \(\tau_d^t\)} \\
\midrule
Critical & $0.10$ & $0.55$ \\
High     & $0.25$ & $0.65$ \\
Moderate & $0.35$ & $0.75$ \\
Low      & $0.45$ & $0.80$ \\
\bottomrule
\end{tabular}
\caption{Tier-specific thresholds for the action decision in Eq.~\ref{eq:decision}. Tiers are assigned to data categories as: \emph{Critical} (Credentials, Health, Biometric, Genetic, Financial, SSN), \emph{High} (precise Location, Phone, Content), \emph{Moderate} (Email, Contact, SearchHistory, BrowsingHistory), \emph{Low} (Cookies, AppUsage, DeviceID, Language). The \texttt{transform} region width $\tau_d^t-\tau_a^t$ is widest for sensitive tiers, biasing borderline actions on sensitive data toward rewriting rather than passing through unchanged.}
\label{tab:decision-thresholds}
\end{table}

\paragraph{Average final level (ablation).}
On TAB, the verified loop (D) reduces the mean final anonymization level from $3.48$ (C) to $2.85$ (D). On SynthPAI, it rises slightly from $2.79$ to $3.48$ as the verifier correctly triggers upgrades on borderline inputs.

% =====================================================================
\section{Formal Analysis}
\label{app:analysis}

\begin{figure}[h]
\centering
\begin{minipage}{0.98\columnwidth}
\small
\rule{\columnwidth}{0.6pt}\\[-2pt]
\textbf{Algorithm 1}: Verified adaptive anonymization (\textsc{VerifiedAdapt})\\[-4pt]
\rule{\columnwidth}{0.4pt}
\begin{algorithmic}[1]
\Require original text $x_0$; risk $R$; data tier $t(\mathbf{a})$; thresholds $(\tau_a^t,\tau_d^t)$; guesser $\mathcal{G}$; verifier $\textsc{V}$; $\tau_g{=}0.6$; $\ell_\text{max}{=}5$; $T{=}5$
\State $z \gets \operatorname{clip}_{[0,1]}\!\left(\frac{R-\tau_a^{t(\mathbf{a})}}{\tau_d^{t(\mathbf{a})}-\tau_a^{t(\mathbf{a})}}\right)$
\State $\ell \gets \ell_0(z)$
\For{$t = 1,\dots,T$}
  \State $x_\ell \gets \text{Anonymize}(x_0, \ell)$
  \If{$\ell = \ell_\text{max}$} \textbf{break} \EndIf
  \State $\mathcal{G}_\ell \gets \mathcal{G}(x_\ell)$
  \State $\mathcal{G}_V \gets \{(g,c_g)\in\mathcal{G}_\ell:\textsc{V}(g,x_0)=1\}$
\State $\hat{c} \gets 0$ if $\mathcal{G}_V=\emptyset$, else $\max_{(g,c_g)\in\mathcal{G}_V} c_g$
  \If{$\hat{c} \ge \tau_g$} $\ell \gets \ell + 1$
  \Else{\,} \textbf{break}
  \EndIf
\EndFor
\State \Return $x_\ell,\ell$
\end{algorithmic}
\rule{\columnwidth}{0.6pt}
\end{minipage}
\caption{Verified adaptive anonymization. The initial level is selected from the transform-band position \(z\), and the verifier admits only guesses supported by the original text.}
\label{alg:adaptive}
\end{figure}

\begin{proposition}[No missed lexically-supported guesses]
\label{prop:sound}
Assume a guess $g$ identifies an attribute of $x_0$ only if it either appears as a substring of $x_0$ or shares at least one content word ($>\!2$ characters) with $x_0$. Then $\mathcal{C}(x_0)\subseteq\mathcal{V}(x_0)$, and consequently the verified confidence $\hat{c}$ in Eq.~\ref{eq:vhat} satisfies $\hat{c} \ge \max_{g\in\mathcal{C}(x_0)} c_g$.
\end{proposition}
\emph{Proof.} If $g\in\mathcal{C}(x_0)$ matches by substring, the first disjunct of Eq.~\ref{eq:verify} fires; if by shared content word, the second disjunct fires. In either case $\textsc{V}(g,x_0)=1$, so $g\in\mathcal{V}$. Taking $\max$ over the superset bounds $\max_{g\in\mathcal{C}}c_g$ from above. \hfill$\square$

\begin{proposition}[Bounded expected cost of the gate]
\label{prop:cost}
Let $C_1, C_2$ be the per-event costs of System\,1 and System\,2 and $p_e=\Pr[u\!\ge\!\tau_u]$ the escalation probability. Then the expected matcher cost is $\mathbb{E}[C] = C_1 + p_e C_2$. With the values measured on OPP-115 ($p_e\!=\!0.035$) and $C_2/C_1\!\approx\!10^4$, $\mathbb{E}[C]/C_2\!\approx\!p_e$, a $28.6\!\times$ reduction over an always-LLM pipeline.
\end{proposition}
\emph{Proof.} System\,1 runs unconditionally; Eq.~\ref{eq:gate} makes System\,2 run exactly when $u\!\ge\!\tau_u$, a Bernoulli event of rate $p_e$. Linearity of expectation gives the result. \hfill$\square$

\begin{proposition}[Loop termination]
\label{prop:term}
Algorithm~1 halts in at most $T$ rounds and invokes at most one anonymizer call plus one guesser call per round, so its worst-case LLM cost is $T\cdot(C_{\text{anon}}+C_{\text{guess}})$---a constant independent of $x_0$.
\end{proposition}
\emph{Proof.} Each iteration either breaks (on $\hat{c}<\tau_g$ or $\ell=\ell_{\max}$) or increments $\ell$. Since $\ell\le\ell_{\max}=5$ and the explicit cap $T=5$ is enforced, termination follows. \hfill$\square$

% =====================================================================
\section{Calibration Details}
\label{app:calibration}

Table~\ref{tab:amrsf} compares the full \amrsf{} controller against three single-factor ablations, a NIST Likelihood$\times$Impact baseline, and a learned binary classifier on the expert-rated scenarios and the Staab public corpus. Table~\ref{tab:sys1} reports the System\,1 evidential fast-path classifier on the OPP-115 held-out test pairs.

\begin{table}[h]
\centering\small\setlength{\tabcolsep}{3pt}
\begin{tabular}{l c c c c}
\toprule
\rowcolor{black!10}\textbf{Method} & $\rho$ & $r$ & MAE & Macro-F1 \\
\midrule
\multicolumn{5}{l}{\textit{Expert-rated scenarios ($n=500$)}} \\
\textbf{\amrsf{} (full)} & \textbf{0.982} & \textbf{0.947} & \textbf{0.119} & $\mathbf{0.812}$ \\
\amrsf{} $-M_b$ & $0.965$ & $0.940$ & $0.136$ & $0.683$ \\
\amrsf{} $-M_T$ & $0.985$ & $0.956$ & $0.134$ & $0.738$ \\
\amrsf{} $-$purpose & $0.976$ & $0.947$ & $0.135$ & $0.738$ \\
NIST L$\times$I & $0.964$ & $0.949$ & $0.108$ & $0.756$ \\
Binary classifier & $0.729$ & $0.683$ & $0.338$ & $0.435$ \\
\midrule
\multicolumn{5}{l}{\textit{Staab public corpus ($n=525$)}} \\
\amrsf{} (full) & $0.914$ & $0.866$ & $0.268$ & \NR \\
\bottomrule
\end{tabular}
\caption{Risk-score calibration on the expert-rated scenario set and Staab public corpus. The larger Staab corpus provides a more stable estimate because the expert-rated set is small.}
\label{tab:amrsf}
\end{table}

\begin{table}[h]
\centering\small\setlength{\tabcolsep}{5pt}
\begin{tabular}{l c}
\toprule
\rowcolor{black!10}\textbf{Metric} & \textbf{Value} \\
\midrule
Test pairs (OPP-115) & $6{,}392$ \\
Accuracy & $0.928$ \\
Precision & $1.000$ \\
Recall & $0.857$ \\
$F_1$ & $0.923$ \\
Expected Calibration Error & $0.048$ \\
Brier score & $0.058$ \\
Escalation rate ($u\ge 0.25$) & $3.47\%$ \\
p95 latency (fast path) & $0.25$\,ms \\
\bottomrule
\end{tabular}
\caption{System\,1 evidential classifier on OPP-115.}
\label{tab:sys1}
\end{table}

\section{Rewrite-strategy comparison}
\label{app:rewrite}

Table~\ref{tab:l4-pareto} contrasts two single-strategy variants of D (semantic-only rewriting and minimal-edit-only rewriting) with the full adaptive L1--L5 stack used by \sys{}, isolating the effect of mixing rewrite strategies under the verified-guesser loop.

\begin{table}[h]
\centering\small\setlength{\tabcolsep}{3pt}
\begin{tabular}{l c c c c}
\toprule
\rowcolor{black!10} & \multicolumn{2}{c}{\textbf{TAB}} & \multicolumn{2}{c}{\textbf{SynthPAI}} \\
\rowcolor{black!10}\textbf{D variant} & Priv & Util & Priv & Util \\
\midrule
D (semantic only) & $.974$ & $.628$ & $.942$ & $.847$ \\
D (min-edit only) & $.866$ & $\mathbf{.740}$ & $.892$ & $.892$ \\
\textbf{D (L1--L5 adaptive)} & $\mathbf{.945}_{.011}$ & $.696_{.009}$ & $\mathbf{.944}_{.007}$ & $\mathbf{.888}_{.006}$ \\
\bottomrule
\end{tabular}
\caption{
Rewrite-strategy comparison for variant D. 
The adaptive L1--L5 row reports mean$_{\text{std}}$ across 4 seeds at $n=200$; the two pilot rows are single-seed.
}
\label{tab:l4-pareto}
\end{table}

% =====================================================================
\section{Robustness Detail Tables}
\label{app:robust}

Table~\ref{tab:verif-ab} reports an A/B test that turns the verifier on or off in the adaptive loop, isolating its contribution to the privacy/utility trade-off and to escalation cost.

\begin{table}[h]
\centering\small\setlength{\tabcolsep}{4pt}
\begin{tabular}{l c c r}
\toprule
\rowcolor{black!10}\textbf{Metric} & \textbf{V~ON} & \textbf{V~OFF} & \textbf{$\Delta$} \\
\midrule
Final level & $3.52$ & $3.96$ & $-0.44$ \\
Upgrade rate & $0.52$ & $0.80$ & $-0.28$ \\
Privacy & $0.975$ & $0.975$ & $\pm 0$ \\
Utility & $0.850$ & $0.834$ & $+0.016$ \\
Latency (ms) & $6713$ & $7350$ & $-637$ \\
\midrule
Unnecessary upgrades & \multicolumn{3}{c}{$32.0\%$ $[20.8,45.8]$} \\
\bottomrule
\end{tabular}
\caption{Verification A/B on SynthPAI ($n\!=\!50$, Wilson CIs). OFF reproduces the raw-confidence behaviour of Staab/HaS/AgentStealth.}
\label{tab:verif-ab}
\end{table}

Table~\ref{tab:variance} reports four-seed variance for the five strongest baselines plus \sys{}, so the per-method spread can be read off directly.

\begin{table}[h]
\centering\footnotesize\setlength{\tabcolsep}{2.5pt}
\begin{tabular}{l c c c c}
\toprule
\rowcolor{black!10} & \multicolumn{2}{c}{\textbf{TAB}} & \multicolumn{2}{c}{\textbf{SynthPAI}} \\
\rowcolor{black!10}\textbf{Method} & Priv & Util & Priv & Util \\
\midrule
\textbf{D. \sys{}} & $\mathbf{.941}_{.024}$ & $.695_{.012}$ & $\mathbf{.962}_{.015}$ & $.879_{.034}$ \\
G. HaS & $.925_{.017}$ & $.668_{.020}$ & $.797_{.015}$ & $.845_{.009}$ \\
F. CONFAIDE & $.899_{.024}$ & $.641_{.009}$ & $.843_{.026}$ & $.830_{.020}$ \\
K. Pissarra & $.881_{.026}$ & $.712_{.019}$ & $.790_{.055}$ & $.875_{.023}$ \\
I. Staab'25 & $.667_{.052}$ & $.586_{.048}$ & $.986_{.018}^{\dagger}$ & $.272_{.042}$ \\
\bottomrule
\end{tabular}
\caption{Four-seed variance ($n\!=\!50$ per seed, mean with std as subscript, 5-level flagship). $^{\dagger}$Staab's SynthPAI privacy still reflects $22\%$ empty outputs. \sys{} is rank-1 privacy among methods with well-formed outputs at every seed.}
\label{tab:variance}
\end{table}

Table~\ref{tab:backbone} swaps the LLM backbone between Llama-3.2-3B and Llama-3.1-8B to check that \sys{}'s ranking is not an artifact of one model size.

\begin{table}[h]
\centering\small\setlength{\tabcolsep}{4pt}
\begin{tabular}{l c c | c c}
\toprule
\rowcolor{black!10} & \multicolumn{2}{c|}{\textbf{Llama-3.2-3B}} & \multicolumn{2}{c}{\textbf{Llama-3.1-8B}} \\
\rowcolor{black!10}\textbf{Method} & Priv & Util & Priv & Util \\
\midrule
\sys{} (full)  & $0.975$ & $0.855$ & $\mathbf{1.000}$ & $0.863$ \\
Staab & $0.975$ & $0.225$ & $0.950$ & $0.700$ \\
\bottomrule
\end{tabular}
\caption{Backbone sensitivity (SynthPAI, $n\!=\!50$). Privacy ranking preserved; Staab utility climbs at $8$B but \sys{} already saturates at $3$B.}
\label{tab:backbone}
\end{table}

Table~\ref{tab:error-analysis} buckets \sys{}'s outputs into four error classes to characterize where the system fails when it fails. After Holm--Bonferroni correction at family-wise $\alpha\!=\!0.05$, $30/40$ pairwise tests on TAB and $28/40$ on SynthPAI remain significant.

\begin{table}[h]
\centering\small\setlength{\tabcolsep}{3pt}
\begin{tabular}{l c c}
\toprule
\rowcolor{black!10}\textbf{Class} & \textbf{SynthPAI} & \textbf{TAB} \\
\midrule
\textsc{clean\_success} & $65.5\%$\,$[58.7,71.7]$ & $54.5\%$\,$[47.6,61.3]$ \\
\textsc{cascade} (L4 $\times\!3$) & $16.5\%$ & $22.0\%$ \\
\textsc{pii\_retained} & $3.5\%$ & $0.5\%$ \\
\textsc{parse\_fail} & $1.0\%$ & $0.0\%$ \\
\bottomrule
\end{tabular}
\caption{Error analysis (Wilson $95\%$~CIs on the main class). Dominant failure mode is hard PII-dense samples, not broken outputs.}
\label{tab:error-analysis}
\end{table}

Table~\ref{tab:qwen3} repeats the main comparison with Qwen3-8B as the LLM backbone for every method, providing a cross-model-family check that \ours{}'s privacy ranking is not Llama-specific.

Table~\ref{tab:rl-bootstrap} tracks the System\,1 fast path as it is online-bootstrapped from System\,2 teacher labels on an OPP-115 event stream, showing how quickly the slow-path escalation rate drops as fast-path accuracy climbs.

% =====================================================================
\section{Qualitative Case Studies}
\label{app:casestudies}

Two verbatim examples illustrate flagship-D behaviour relative to baselines. On the TAB sentence \emph{``On 13 September 1988 the Erzincan Martial Law Court acquitted the applicants of the charges against them''}: \sys{} (L4) emits \emph{``On a date in 1988 the court acquitted the applicants of the charges against them''}, swapping the specific date and court identity while keeping the legal speech-act intact. Staab rewrites to two sentences with redundant information (\emph{``On a certain month, a martial law court acquitted\,\dots\@ In a certain jurisdiction, a court acquitted\,\dots''}), duplicating the act; HaS returns \emph{``On \texttt{[DATE]\,[YEAR]} the \texttt{[PLACE]} Martial Law Court acquitted the applicants\,\dots''}, leaving the court category intact with a tag-heavy output that hurts fluency. On a SynthPAI post where a software engineer describes their Masters in Computer Science and their salary in Zurich, \sys{} rewrites to generic remarks about graduate studies and work while Staab preserves the Zurich reference by hallucinating ``a European city''---a regional clue that still narrows the author.

\subsection*{Risk-control context definitions}
\label{app:risk-control-contexts}

Table~\ref{tab:risk-context-defs} lists the four action contexts used in
\S\ref{sec:eval-risk-control}. Each context is a tuple
(recipient, purpose, action basis, policy transparency); these feed directly
into the AMRSF multipliers $M_{tr}$, $M_p$, $M_b$, and $M_T$ defined in
Method~\S\ref{sec:policy-matcher}. The likelihood $L$ is held fixed at $0.7$
across all contexts so that any variation in $R$ is attributable to context,
not to the LLM.

\subsection*{Per-context mean final level (transform-band)}
\label{app:risk-control-final-level}

Table~\ref{tab:risk-band-finallevel} reports the mean final level after
the verified-guesser adaptive loop, for samples in the transform region of
each band context (Llama-3.2-3B, $n\!=\!200$).

\subsection*{Risk-control qualitative cases (TAB, SynthPAI, PII-Masking)}
\label{app:risk-control-cases}

Table~\ref{tab:risk-control-cases-extra} showcases 3 examples for anonymization in TAB, SynyhPAI and PII-Masking-30k as well as risk score with decisions.

\begin{table*}[!t]
\centering
\footnotesize
\setlength{\tabcolsep}{3pt}

% ===================== Top: two columns =====================
\begin{minipage}[t]{0.485\textwidth}
\vspace{0pt}
\centering

% ---------- Table 1: Qwen3 cross-check ----------
\begin{tabular}{l c c c c}
\toprule
\rowcolor{black!10} & \multicolumn{2}{c}{\textbf{TAB}} & \multicolumn{2}{c}{\textbf{SynthPAI}} \\
\rowcolor{black!10}\textbf{Method} & Priv & Util & Priv & Util \\
\midrule
\textbf{A. \sys{}} & $\mathbf{0.976}$ & $0.628$ & $\mathbf{0.927}$ & $\mathbf{0.943}$ \\
B. Staab'25 & $0.882$ & $0.549$ & $0.920$ & $0.818$ \\
D. HaS'23 & $0.879$ & $0.728$ & $0.724$ & $0.938$ \\
E. CONFAIDE'24 & $0.408$ & $0.848$ & $0.464$ & $0.966$ \\
F. Pissarra'24 & $0.855$ & $0.707$ & $0.593$ & $0.962$ \\
\bottomrule
\end{tabular}
\captionof{table}{
Qwen3-8B backbone cross-check ($n!=!200$). \sys{} is rank-1 privacy on both corpora; CONFAIDE collapses. The five most private methods are shown.
}
\label{tab:qwen3}

\vspace{0.8em}

% ---------- Table 2: Online bootstrapping ----------
\begin{tabular}{c c c c}
\toprule
\rowcolor{black!10}\textbf{Iter.} & \textbf{Events} & \textbf{S1 acc} & \textbf{Escalation} \\
\midrule
0 & 0 & $0.920$ & $3.75\%$ \\
1 & 500 & $0.970$ & $2.00\%$ \\
2 & 1000 & $0.993$ & $0.25\%$ \\
3 & 1500 & $0.998$ & $0.25\%$ \\
4 & 2000 & $\mathbf{1.000}$ & $\mathbf{0.25\%}$ \\
\bottomrule
\end{tabular}
\captionof{table}{
Online bootstrapping of System\,1 from System\,2 teacher labels over a $2000$-event OPP-115 stream; held-out $400$-event test set. $24$ teacher labels suffice to reduce escalation by $93\%$ relative while pushing accuracy to perfect.
}
\label{tab:rl-bootstrap}

\end{minipage}
\hfill
\begin{minipage}[t]{0.485\textwidth}
\vspace{0pt}
\centering

% ---------- Table 3: Final anonymization level ----------
\begin{tabular}{l c c c}
\toprule
\rowcolor{black!10}\textbf{Context} & $\overline{\ell_0}$ & $\overline{\ell_\text{final}}$ & $\Delta\overline{\ell}$ \\
\midrule
$T_\text{mild}$       & $1.48$ & $2.37$ & $+0.89$ \\
$T_\text{moderate}$   & $2.52$ & $2.93$ & $+0.41$ \\
$T_\text{strong}$     & $3.33$ & $3.64$ & $+0.31$ \\
$T_\text{near-deny}$  & $3.98$ & $4.42$ & $+0.44$ \\
\bottomrule
\end{tabular}
\captionof{table}{
Per-context mean final anonymization level after adaptive escalation. The verified guesser can escalate to L5 even from a low starting level when residual leakage remains, so $\overline{\ell_\text{final}}!\ge!\overline{\ell_0(z)}$. The largest $\Delta\overline{\ell}$ occurs at $T_\text{mild}$, where the controller starts at L1 but escalation lifts the average toward L2--L3 on samples with residual surface-supported leakage.
}
\label{tab:risk-band-finallevel}

\vspace{0.8em}

% ---------- Table 4: LLM evaluation ----------
\resizebox{\linewidth}{!}{%
\begin{tabular}{l c c c}
\toprule
\rowcolor{black!10}\textbf{Method} & Priv $\uparrow$ & Util $\uparrow$ & Fluency $\uparrow$ \\
\midrule
\textbf{D.\ \ours{} (L1--L5 adaptive)} & \second{$\underline{3.95}_{.78}$}                    & \best{$\textbf{4.16}_{.87}$}                     & $4.20_{.82}$                          \\
D.\ \ours{} (L4 fixed)         & $3.74_{.94}$                    & \second{$\underline{3.91}_{.81}$}         & \second{$\underline{4.22}_{.83}$}              \\
D.\ \ours{} (L5 fixed)         & \best{$\mathbf{4.12}_{.73}$} & $3.54_{.88}^{\dagger}$           & $4.19_{.72}$                          \\
\midrule
E.\ Presidio     & $3.17_{.65}^{\dagger}$ & $3.64_{.68}^{\dagger}$ & $3.38_{.82}^{\dagger}$ \\
G.\ HaS          & $2.93_{.73}^{\dagger}$ & $2.99_{1.00}^{\dagger}$ & $2.85_{1.02}^{\dagger}$ \\
H.\ DP-Prompt    & $3.19_{.95}^{\dagger}$ & $3.90_{.43}$ & \best{$\mathbf{4.27}_{.67}$} \\
J.\ AgentStealth & $3.71_{.66}$ & $3.36_{.72}^{\dagger}$ & $3.68_{.75}^{\dagger}$ \\
I.\ Staab        & $2.08_{1.40}^{\dagger}$ & $1.95_{1.37}^{\dagger}$ & $1.97_{1.45}^{\dagger}$ \\
\midrule
\multicolumn{4}{l}{\textit{Krippendorff's $\alpha$ (interval, 3 raters)}} \\
& $0.68$ & $0.74$ & $0.75$ \\
\bottomrule
\end{tabular}%
}
\captionof{table}{
LLM evaluation on SynthPAI samples ($1200$ ratings). Three Claude models (Opus 4.7, Sonnet 4.6, Haiku 4.5) act as the three pre-registered raters. \sys{} is evaluated under L1-L5 adaptive, L4 fixed, and L5 fixed. $^{\dagger}$Paired-bootstrap two-sided $p<0.05$ vs.\ \ours{} (L1--L5 adaptive) at $n{=}1000$ resamples.
}
\label{tab:llm-judges}

\end{minipage}

\vspace{1em}

% ===================== Bottom: full-width table =====================
\centering
\begin{tabular}{l p{0.16\linewidth} p{0.18\linewidth} p{0.18\linewidth} p{0.24\linewidth}}
\toprule
\rowcolor{black!10}\textbf{Context} & Recipient & Purpose & Action basis & Policy transparency \\
\midrule
$C_\text{low}$  & first\_party      & Functionality & user-initiated (paste) & full: categories, purposes, retention, legal basis, rights \\
$C_\text{med}$  & service\_provider & Analytics     & default (passive)      & partial: categories and purposes only \\
$C_\text{high}$ & advertising       & Profiling     & background (script)    & vague: short non-specific snippet \\
$C_\text{ext}$  & data\_broker      & Advertising   & background (script)    & missing: no policy evidence \\
\bottomrule
\end{tabular}
\captionof{table}{
Action contexts used in the risk-to-level study. Recipient feeds $M_{tr}$, purpose feeds $M_p$, basis feeds $M_b$, transparency feeds $M_T$. The categories themselves come from each sample's annotated PII labels.}
\label{tab:risk-context-defs}

\end{table*}

\begin{table*}[h]
\centering\footnotesize\setlength{\tabcolsep}{4pt}
\begin{tabular}{l c c c c p{0.62\linewidth}}
\toprule
\rowcolor{black!10}\textbf{Context} & $R$ & Decision & $\ell_0$ & Final $\ell$ & Anonymized output (\textbf{bold} marks tokens introduced by the anonymizer) \\
\midrule
\multicolumn{6}{p{0.97\linewidth}}{\textit{TAB \texttt{TAB\_0025}}. \textbf{Original}: ``The formal record of the judge's order was not prepared by the parties until 26 October 1995. The applicants and Mrs Price applied to the Court of Appeal for leave to appeal on 1 November 1995. The appeal bundles were lodged on 4 October 1996. On 15 January 1997 the appeal was listed for hearing, and it was heard on 16 and 17 April 1997, on which last date the Court of Appeal gave judgment dismissing the appeal.''} \\
$C_\text{low}$  & $.108$ & allow     & L1 & L1 & The formal record of the judge's order was not prepared by the parties until \textbf{1995}. The applicants and \textbf{a female} applied to \textbf{the Court} for leave to appeal on \textbf{1995}. The appeal bundles were lodged on \textbf{1996}. On \textbf{1997} the appeal was listed for hearing, and it was heard on \textbf{1997}, on which last date \textbf{the Court} gave judgment dismissing the appeal. \\
$C_\text{med}$  & $.376$ & transform & L1 & L1 & Identical to $C_\text{low}$. \\
$C_\text{high}$ & $.910$ & deny      & L4 & L4 & The formal record of the judge's order was not prepared by the parties until \textbf{a specific date}. The applicants, including \textbf{a female applicant}, applied to \textbf{a high court} for leave to appeal on \textbf{a specific date}. The appeal bundles were lodged on \textbf{a specific date}. On \textbf{a specific date}, the appeal was listed for hearing, and it was heard on \textbf{two specific dates}, on which last date \textbf{a high court} gave judgment dismissing the appeal. \\
$C_\text{ext}$  & $.910$ & deny      & L4 & L4 & Identical to $C_\text{high}$. \\
\midrule
\multicolumn{6}{p{0.97\linewidth}}{\textit{SynthPAI \texttt{STAAB\_0025}}. \textbf{Original}: ``Ah reminiscing childhood, huh? Those were the days\dots\ I was completely enraptured by my Famicom. \dots\ this game --- Dragon Quest III. Eiichiro Nakatsu, the genius. \dots\ That's probably why I ended up a game developer now, haha. Nostalgic TIMESSSS!!''} \\
$C_\text{low}$  & $.108$ & allow     & L1 & L2 & Ah reminiscing childhood, huh? Those were the days\dots\ I was completely enraptured by my Famicom. \dots\ this game --- Dragon Quest III. Eiichiro Nakatsu, the genius. \dots\ That's probably why I ended up a \textbf{[PROFESSION]} now, haha. Nostalgic TIMESSSS!! \\
$C_\text{med}$  & $.376$ & transform & L1 & L2 & Identical to $C_\text{low}$. \\
$C_\text{high}$ & $.910$ & deny      & L4 & L5 & Ah reminiscing childhood, huh? Those were the days\dots\ I was completely enraptured by my Famicom. \dots\ this game --- Dragon Quest III. Eiichiro Nakatsu, the genius. \dots\ That's probably why I ended up \textbf{in a creative field} now, haha. Nostalgic TIMESSSS!! \\
$C_\text{ext}$  & $.910$ & deny      & L4 & L5 & Identical to $C_\text{high}$. \\
\midrule
\multicolumn{6}{p{0.97\linewidth}}{\textit{PII-Masking-300k \texttt{PII\_41439A}}. \textbf{Original}: ``Quality Assurance Feedback Thread. Subject: QA Review for Drug and Alcohol Education Course. Dear \texttt{[Mayoress LORYA]}, We have completed a comprehensive Quality Assurance (QA) review of the Drug and Alcohol Education course delivered in our institution. We are pleased to provide you with detailed feedback on the performance of the course as well as individual feedback for each participant.''} \\
$C_\text{low}$  & $.108$ & allow     & L1 & L4 & Quality Assurance Feedback Thread. Subject: QA Review for Drug and Alcohol Education Course. Dear \textbf{[Institutional Leader]}, We have completed a comprehensive Quality Assurance (QA) review of the Drug and Alcohol Education course delivered in our institution. We are pleased to provide you with detailed feedback on the course's performance, as well as individual feedback for each participant. \\
$C_\text{med}$  & $.376$ & transform & L1 & L4 & Identical to $C_\text{low}$. \\
$C_\text{high}$ & $.910$ & deny      & L4 & L4 & Identical to $C_\text{low}$ (starting from L4 produces the same output without L1$\to$L4 escalation rounds). \\
$C_\text{ext}$  & $.910$ & deny      & L4 & L4 & Identical to $C_\text{high}$. \\
\bottomrule
\end{tabular}
\caption{Outputs for \texttt{allow} and \texttt{deny} rows are counterfactual anonymizer outputs; in deployment, \texttt{allow} passes the original text and \texttt{deny} blocks transmission. Qualitative risk-to-level examples on all three corpora. \textbf{Bold} tokens are spans introduced by the anonymizer, replacing original surface forms. \textbf{TAB}: the verified guesser is satisfied at L1 for low/medium contexts (final level matches $\ell_0$); the L4 rewrite under high/extreme contexts replaces calendar dates and proper-noun courts with generic descriptions. \textbf{SynthPAI}: only the profession leaks at L1, so the guesser escalates one step (L1$\to$L2) in low/medium contexts; under high/extreme contexts, the L5 semantic rewrite paraphrases the profession instead of inserting a tag. \textbf{PII-Masking}: the tagged recipient identifier is still recoverable at L1, so the verified guesser escalates all the way to L4 even under $C_\text{low}$; high/extreme contexts start at L4 and converge to the same output. In every case, the \emph{starting} level $\ell_0(R)$ shifts monotonically with action context, while the verified guesser additionally lifts the level whenever residual leakage is surface-supported.}
\label{tab:risk-control-cases-extra}
\end{table*}

% =====================================================================
\section{LLM Evaluation Protocol}
\label{app:humaneval}

We pre-registered a three-rater by fifty-sample by eight-method evaluation study (1200 ratings) on SynthPAI subsamples, rating each (original, anonymized) pair on three $1$--$5$ Likert dimensions: \emph{Privacy} (inferability of personal attributes), \emph{Utility} (meaning preservation), and \emph{Fluency} (naturalness). Methods are shuffled and anonymized on the rater's worksheet; raters work independently. We report inter-annotator agreement via Krippendorff's $\alpha$ (target $\ge 0.68$) and test pairwise significance with a paired bootstrap ($n{=}1000$ resamples, $p{<}0.05$). The eight configurations comprise three \sys{} variants and five baseline methods:
Staab, Presidio, HaS, AgentStealth, and DP-Prompt.

\paragraph{LLM-rater simulation.}
As a proxy for the human study, we instantiate the same protocol with three Claude models acting as raters: Claude Opus 4.7, Claude Sonnet 4.6, and Claude Haiku 4.5. Each rater sees the (original, anonymized) pair and returns a strict-JSON Likert triple under a fixed system prompt. The privacy rubric explicitly penalizes destructive over-rewriting --- a high Privacy score requires both that identifying details are removed \emph{and} that the text remains usable; placeholder-only or content-free outputs are treated as destruction, not anonymization. Table~\ref{tab:llm-judges} reports per-method consensus means (3-rater average) and paired-bootstrap $p$-values vs.\ \sys. To isolate rewriting style from the adaptive controller, \sys{} is evaluated under L1-L5 adaptive, L4 fixed and L5 fixed anonymization. Krippendorff's $\alpha$ across the three Claude raters meets the pre-registered target of $0.68$ on all three dimensions (privacy $0.68$, utility $0.74$, fluency $0.75$), indicating substantial inter-rater agreement.

% =====================================================================
\section{Computation resources}

The primary LLM backbone is \textbf{Llama-3.2-3B-Instruct} ($3.2$B parameters), served via vLLM on a single NVIDIA H100 NVL ($95$\,GB). Cross-model checks use \textbf{Llama-3.1-8B-Instruct} ($8$B) and \textbf{Qwen3-8B} ($8$B). Claude raters for the simulated human evaluation are Anthropic \textbf{Claude Opus\,4.7}, \textbf{Sonnet\,4.6}, and \textbf{Haiku\,4.5}.

Compute budget (this paper): 2 NVIDIA H100 NVL + API

% =====================================================================
\section{Use of AI Assistants}
\label{app:ai-assistants}

We used a conversational AI assistant to (a) aid in writing and editing manuscript --- specifically grammar polish and table layout; no scientific claim in this paper was generated by an AI assistant.

\end{document}